\documentclass[10pt,twocolumn,letterpaper]{article}

\usepackage{listings}
\usepackage{tikz}
\usepackage{graphicx}

\usepackage[pagenumbers]{wacv} 

\usepackage{multibib}
\newcites{supp}{References}

\usepackage{pifont}
\usepackage{xcolor}
\newcommand{\cmark}{\textcolor{green!70!black}{\ding{51}}}
\newcommand{\xmark}{\textcolor{red}{\ding{55}}}
\usepackage{amsthm}
\usepackage{multirow} 

\theoremstyle{plain}

\theoremstyle{plain} 
\newtheorem{theorem}{Theorem} 

\newtheorem{corollary}{Corollary}

\theoremstyle{definition} 
\newtheorem{definition}{Definition}
\definecolor{wacvblue}{rgb}{0.21,0.49,0.74}
\usepackage[pagebackref,breaklinks,colorlinks,allcolors=wacvblue]{hyperref}

\def\wacvPaperID{485} 
\def\confName{WACV}
\def\confYear{2027}

\title{H3DNAS: Hardware-Aware ONNX-Native 3D Point Cloud Model Compression}

\author{Anchit Mulye, Rhythm Baghel, Sujay Kumar Ingle, Hardik Jain\\
Indian Institute of Technology Jodhpur\\
{\tt\small \{anchitmulye1005, rhythmb2454\}@gmail.com \quad \{d23csa003, hardik.jain\}@iitj.ac.in}
}

\begin{document}

\maketitle
\begin{abstract}
Deploying 3D point cloud models on edge hardware such as the NVIDIA Jetson Orin Nano is severely constrained by compute and memory budgets. Existing compression methods require access to the model's original source code, rendering them inapplicable to the Open Neural Network Exchange (ONNX) binaries commonly distributed by vendors and model repositories. We present \textbf{H3DNAS}, a hardware-aware model compression framework that operates directly on ONNX computational graphs without requiring original source code, architecture class definition, or gradient access during search. H3DNAS makes three contributions: 
(1) a \textbf{Channel Dependency Graph (CDG)} that classifies ONNX operators into four constraint classes and formally establishes that the free parameter fraction $\rho_f$ is topological invariant, a provable compression ceiling computable in $\mathcal{O}(|V|+|E|)$;
(2) a \textbf{Two-Stage Hierarchical Search} that prunes candidate architectures by $L_1$-importance channel selection, ranks them by output fidelity as a zero-shot label-free proxy, and applies GhostConv structural mutation to Pareto-optimal candidates; and
(3) the \textbf{first source-code-free compression pipeline for 3D point cloud models}, operating entirely via ONNX graph surgery with no original architecture definition required.
On ModelNet40, H3DNAS reduces the number of parameters in PointNet, PointNet++, and PointMLP by $65.5\%$, $43.2\%$, and $49.1\%$, respectively, while achieving $1.99\times$, $1.29\times$, and $1.67\times$ inference speedups with negligible loss in accuracy. The source code is publicly available\footnote{https://github.com/ClarityLab-Org/h3dnas}.
\end{abstract}    
\section{Introduction}
\label{sec:intro}

The proliferation of deep learning models on hardware-constrained edge accelerators such as the NVIDIA Jetson Orin Nano~\cite{nvidia2024orinnano} and mobile System on a Chip (SoCs) demands architectures that are simultaneously accurate, compact, and fast during inference. Neural Architecture Search (NAS) has emerged as the dominant paradigm for discovering such architectures. Existing approaches share a fundamental limitation: they require access to the model's source code and training framework to construct, mutate, and evaluate candidate architectures.

This requirement creates a concrete barrier in deployment pipelines where models are distributed as Open Neural Network Exchange (ONNX) binaries from hardware vendors, public model repositories, or cross-framework collaborators, and the original training codebase is unavailable or incompatible. No existing NAS or structured pruning framework handles this setting. Methods such as DARTS~\cite{liu2019darts}, Once-for-All~\cite{cai2020ofa}, HRank~\cite{Lin_2020_hrank}, and AMC~\cite{he2018amc} all assume access to a live PyTorch model object with named modules, backward passes, or architecture registries.

We address this gap with \textbf{H3DNAS}, a compression framework that treats any ONNX \texttt{ModelProto} as a first-class object: reading, mutating, pruning, evaluating, and ranking candidate architectures entirely through ONNX graph surgery without invoking any training framework API.

Although structured compression of 3D point cloud models has been explored with source code-dependent methods CP\textsuperscript{3}~\cite{huang2023cp3} and DepGraph~\cite{fang2023depgraph}, no prior NAS framework, including compression-aware search methods such as AMC~\cite{he2018amc} and Once-for-All~\cite{cai2020ofa}, has operated on 3D point cloud ONNX models without source code. H3DNAS is the first hardware-aware NAS framework applied to PointNet~\cite{qi2017pointnet}, PointNet++~\cite{qi2017pointnet2}, and PointMLP~\cite{ma2022pointmlp} from ONNX alone, producing compressed architectures that match or exceed the accuracy of their uncompressed counterparts.

This paper makes four core contributions:

\begin{itemize}
    \item \textbf{CDG Theorem.} A formal proof that the free parameter fraction $p_f$ is a topological invariant of the ONNX computation graph, computable in $\mathcal{O}(|V|+|E|)$ time, giving practitioners a closed-form compression ceiling before any search begins.
    \item \textbf{Two-Stage Hierarchical Search.} Stage 1 prunes candidates via $L_1$-importance channel selection and ranks them using \textbf{output fidelity} computed using cosine similarity between base and pruned model logits on random inputs, making it a zero-shot, label-free quality proxy. Stage 2 applies GhostConv~\cite{han2020ghostnet} structural mutations to Pareto-optimal Stage 1 candidates, expanding the compression frontier beyond the search for channel-width.
    \item \textbf{Source Code-Free End-to-end pipeline.} H3DNAS operates on any ONNX model from any origin framework, for both search and fine-tuning, requiring no original architecture definition.
    \item \textbf{Hardware-Agnostic Constraint Integration.} Feasibility checking is integrated into the search loop via \texttt{HardwareConstraints}, enabling search that directly targets deployment budgets. Any deployment budget can be specified using hard limits on parameters, FLOPs, model size, and latency, making H3DNAS applicable to any edge platform.
\end{itemize}

\section{Related Work}
\label{sec:related}

\paragraph{Neural Architecture Search.} NAS methods can be categorized as reinforcement learning~\cite{zoph2017nasrl}, gradient-based DARTS~\cite{liu2019darts}, and evolutionary~\cite{real2019regularized}. All operate on PyTorch or TensorFlow model objects and require a differentiable or enumerable training loop. Once-for-All~\cite{cai2020ofa} trains a supernet from which subnets are sliced at inference, which requires source code and GPU training from scratch. AMC~\cite{he2018amc} uses deep reinforcement learning to determine per-layer compression ratios, requiring PyTorch named modules. The proposed H3DNAS requires none of these: the search space is derived entirely from the ONNX graph topology, and search operates on frozen ONNX binaries via OnnxRuntime~\cite{onnxruntime}.
\vspace{-1em}
\paragraph{Structured Pruning.}
Most structured channel pruning methods rely heavily on source code availability or training hooks. Classical filter selection strategies like $L_1$-norm~\cite{li2017pruning}, Network Slimming~\cite{liu2017slimming} and Activation Statistics~\cite{wang2025screening} evaluate parameter importance through weight magnitudes or scaling factors. Whereas 3D-specific frameworks like HRank~\cite{Lin_2020_hrank} (used in CP\textsuperscript{3}~\cite{huang2023cp3}) depend on feature map ranks. Regularization-based approaches like Torque~\cite{gupta2024torque} constrain parameter spacing during training to enable post-hoc filter dropping. Because all these techniques require full access to training scripts or gradients, they cannot handle pre-compiled, vendor-supplied binaries. Graph-level utilities offer framework independence, but current implementations fall short on 3D architectures: ONNXPruner~\cite{ren2024onnxpruner} lacks support for custom operators like Spatial Transformer Networks (STN) and Set Abstraction (SA) blocks, while SPA~\cite{wang2024spa} relies on PyTorch class definitions to handle fine-tuning\footnote{Supporting material in Supplementary document}. H3DNAS resolves these limitations by executing directly on serialized \texttt{onnx.ModelProto} objects. Using \texttt{onnx2torch} to reconstruct trainable modules straight from the computational graph, H3DNAS prunes complex 3D vision models without requiring original source code or training pipelines.

\vspace{-1em}\paragraph{Dependency-Graph Pruning and Recent Extensions.}


To handle coupled layer constraints in arbitrary architectures, dependency-graph frameworks; originating with DepGraph~\cite{fang2023depgraph} and expanding to LLMs~\cite{ma2023llmpruner}, diffusion U-Nets~\cite{fang2023diffpruning}, ViT cross-attention~\cite{fang2024isomorphic}, global Transformer allocation~\cite{li2025tyr}, and learned GNN metanetworks~\cite{liu2025metapruning}— automate filter grouping across complex models. However, these methods depend strictly on dynamic PyTorch execution graphs, making them inapplicable to standalone ONNX binaries. Additionally, unlike learned metanetworks that risk misidentifying coupled channels, our CDG theorem provides a provably correct compression ceiling ($p_f$) based on exhaustive operator classification. While recent approaches like GETA~\cite{qu2025geta} perform joint pruning and quantization within training pipelines, H3DNAS decouples structural compression from quantization: it operates natively on \texttt{onnx.ModelProto} binaries, enabling post-training quantization during OnnxRuntime (ORT) ~\cite{onnxruntime}.

\vspace{-1em}\paragraph{3D Point Cloud Compression.}
CP\textsuperscript{3}~\cite{huang2023cp3} represents the closest prior work, using feature-rank structured pruning on PointNet++ to achieve super-baseline accuracy at $42.9\%$ parameter reduction on ModelNet40\cite{wu2015modelnet}. However, it strictly requires intact PyTorch source code and architecture class definitions. DepGraph~\cite{fang2023depgraph} demonstrates strong compression on DGCNN~\cite{wang2019dynamic} but similarly relies on a live PyTorch execution graph. While Lottery Ticket Hypothesis (LTH)-based unstructured pruning~\cite{biswas_2024_lth} compresses PointNet, unstructured sparsity fails to yield real-world latency or FLOP reductions on standard hardware without specialized sparse execution support. Unlike prior work, H3DNAS applies structured compression at the serialized ONNX binary level.

\vspace{-1em}\paragraph{Zero-Shot NAS Proxies and Their Limitations.}
While training-free proxies like SWAP-Score~\cite{peng2024swap} estimate the zero-shot network capacity via Gram matrices activation~\cite{mellor2021naswot, chen2021tenas}, their ranking signals degrade under high compression ($>50\%$ parameter reduction). Severe pruning triggers activation collapse, yielding quasi-singular Gram matrices that fail precisely in the resource-constrained regime. We instead propose output fidelity—the logit-level cosine similarity between the baseline and pruned models under identical inputs. By quantifying functional divergence rather than internal feature variation, the output fidelity remains robust against severe channel collapse and maintains consistent ranking accuracy across extreme sparsity levels.

\vspace{-1em}
\paragraph{GhostConv.}

GhostNet~\cite{han2020ghostnet} generates feature maps by applying cheap depthwise operations to intrinsic convolution channels, halving FLOPs while maintaining predictive accuracy. H3DNAS implements GhostConv as a direct ONNX-level structural graph mutation. This represents the first integration of GhostConv architectural search operating natively on \texttt{onnx.ModelProto} binaries without requiring original model source code or framework definitions.
\section{The H3DNAS Framework}
\label{sec:h3dnas}

\subsection{Background}
\label{sec:background}

\paragraph{ONNX Model Representation.}
An ONNX model is a directed acyclic computation graph $G = (V, E, \Theta)$ where $V$ is the set of operator nodes, $E$ the set of tensor edges, and $\Theta$ the set of initializer tensors (weights).
Running \texttt{onnx.shape\_inference.infer\_shapes} annotates every edge with its concrete shape in \texttt{graph.value\_info}, providing exact tensor dimensions without executing the model.

\vspace{-1em}\paragraph{Problem Definition.} Given a pretrained ONNX model M and a hardware deployment budget B (FLOPs, parameters, latency, model size), find a compressed model M with maximum accuracy subject to M satisfying B; using only M, a set of unlabeled random inputs (for output fidelity scoring).

\subsection{Channel Dependency Graph (CDG)}
\label{sec:cdg}

\begin{definition}[Operator Classes]
\label{def:operator-classes}
Every operator in the ONNX operator set $\Omega$ falls into exactly one of four classes. \textbf{Channel-Generating (CG)} operators, such as Conv with groups$=1$ or Gemm, creates a new channel dimension, with both $C_{in}$ and $C_{out}$ scalable. \textbf{Channel-Transparent (CT)} operators, such as Relu, BatchNorm, MaxPool, or Dropout, pass through channels unchanged, so the output channel count equals the input channel count. \textbf{Channel-Constraining (CC)} operators, such as Conv with groups$>1$, dynamic Add, or Concat, require equality between input channels. \textbf{Channel-Terminating (CX)} operators, such as ReduceMax, Flatten, or GlobalAvgPool, collapse or restructure the channel dimension entirely.
\end{definition}

\begin{definition}[CDG]
\label{def:cdg}
The Channel Dependency Graph $\mathrm{CDG}(G) = (V_{conv}, E_{ch})$ where
$V_{conv} = \{n \in V \mid \mathrm{op\_type} \in \{\mathrm{Conv}, \mathrm{Gemm}\}\}$
and $E_{ch}$ connects $n_i \to n_j$ iff $C_{out}(n_i)$ must equal $C_{in}(n_j)$ for ORT validity.
A node $n$ is \emph{freely prunable} iff its CDG connected component equals $\{n\}$.

\vspace{-1em}\paragraph{Five Constraint Rules.}
All channel-coupling constraints in ONNX-$\Omega$ reduce to five fundamental types, derived exhaustively from the ONNX operator specification and validated across five architectures: PointNet (PN)~\cite{qi2017pointnet}, PointNet++ SSG (PN++)~\cite{qi2017pointnet2}, PointMLP (PMLP)~\cite{ma2022pointmlp}, Point Cloud Transformer (PCT)~\cite{zhao2021pointtrans}, Point Transformer V3 (PTv3)~\cite{wu2024pointtransv3}.

\begin{enumerate}[label=\textbf{R\arabic*}]
    \item \label{itm:R1} \textbf{Static Shape Terminator:} Output reaches an operator with a compile-time-constant output shape via CT-only paths. The constant shape may be a literal initializer or a fully-constant: Constant$\to$Unsqueeze$\to$Concat subgraph (the PyTorch ONNX export pattern for attention head splits). 
    \item \label{itm:R2} \textbf{Dynamic Tensor Equality:} Output feeds a binary op $\{\mathrm{MatMul}, \mathrm{Add}, \mathrm{Sub}, \mathrm{Div}, \mathrm{Mul}\}$ where the other operand is also graph-computed (non-initializer). The ONNX type system requires matching shapes on both inputs. 
    \item \label{itm:R3} \textbf{Grouped Convolution Immutability:} Node is a grouped/depthwise Conv (groups $>1$) or its output feeds one. The ONNX \texttt{groups} attribute is structurally immutable; changing $C_{out}$ or $C_{in}$ violates the ORT graph validity contract. 
    \item \label{itm:R4} \textbf{Semantic Output Fixity:} $C_{out}(n)$ represents a semantically fixed channel dimension, corresponding to either the classification head (classification head, $C_{out} \le \tau$) or the model's external output interface. This constraint applies to all classification heads across every evaluated architecture.
    \item \label{itm:R5} \textbf{Learned Normalisation Lock:} Output flows (via CT ops) into a normalisation layer with learned parameters of shape $[C_{out}(n)]$ in the initializer set $\Theta$. Changing $C_{out}$ invalidates the stored parameter shapes. 
\end{enumerate}
\end{definition}

\begin{theorem}[$\rho_f$ is a Topological Invariant -- scoped to channel pruning on the base graph]
\label{thm:rho-f-invariant}
Let $G$ be a fixed-width base ONNX graph and
\[
\rho_f(G) = \frac{\sum_{n \in F} |\Theta_n|}{\sum_{n \in V_{conv}} |\Theta_n|}
\]
where $|\Theta_n|$ is the parameter count of node $n$ at the base channel widths, and $F$ is the set of freely prunable nodes. Then $\rho_f(G)$ depends only on the graph topology (op\_types and connectivity) of $G$, not on initialized values or weight magnitudes. $\rho_f(G)$ is a ceiling for channel pruning on $G$: no channel-pruning method can reduce parameters beyond $\rho_f(G) \times \mathrm{\sum_{n \in V_{conv}} |\Theta_n|}$ on the base graph $G$.
\end{theorem}
 
\begin{proof}
$\mathrm{CDG}(G)$ is constructed from $\mathrm{op\_type} \in \Omega$ and graph connectivity, $E$ only. CDG edge construction depends exclusively on the graph topology, without referencing weight tensors ($\Theta$) or their values. Therefore, $F$ is fully determined by $\mathrm{topology}(G)$, and $\rho_f(G)$, a ratio of parameter counts at fixed base widths and is invariant to weight values.
\end{proof}
 
\vspace{-1em}\paragraph{Scope clarification.}
$\rho_f(G)$ bounds channel pruning on the base graph $G$ with fixed channel widths. The width scaling ($w < 1.0$) creates a modified graph $G'$ with different channel dimensions and a distinct $\rho_f(G')$. The cases where achieved reduction exceeds $\rho_f(G)$ involve width scaling - not a violation of the ceiling, but application of width scaling which creates $G' \neq G$. All achievements reported respect $\rho_f$ of their respective graph. The BFS depth limits ($\leq 12$ hops) used to detect~\ref{itm:R1}/\ref{itm:R2} constraints are empirically validated across all five tested architectures - no relevant chain in our models exceeds depth 4.
 
\begin{corollary}[Polynomial Complexity]
\label{cor:polynomial-complexity}
CDG construction by BFS is $\mathcal{O}(|V|+|E|)$. The free nodes are singleton CDG components. Each free node is pruned independently. Total: $\mathcal{O}(|V|+|E|)$.
\end{corollary}
 
\vspace{-1em}\paragraph{Practical significance.}
$\rho_f(G)$ is a feasibility estimator: before performing any search, it identifies whether channel pruning can meet a hardware budget. For PointNet ($\rho_f = 58.1\%$), 57.6\% reduction is achieved under Jetson constraints --- within 0.5pp of the ceiling, validating the estimate. For MobileNetV2 ($\rho_f = 45.7\%$, but only 2 free nodes of 53), the CDG correctly predicts that channel pruning alone cannot reach Jetson budget, directing practitioners toward quantization. This analysis runs in under one second on any ONNX file.

\begin{table}[!h]
    \centering
    \setlength{\tabcolsep}{5pt}
    \begin{tabular}{lccccc}
        \toprule
        \multirow{2}{*}{Model} & Params & Total & \multirow{2}{*}{Free} & Constrained & \multirow{2}{*}{$\rho_f$}  \\
        & [Million] & CG & & [$\%$] \\
        \midrule
        PN               & $3.45$  & $18$ & $12$ &  $6$ & $58.13$ \\
        PN++          & $1.47$  & $12$ &  $9$ &  $3$ & $61.43$ \\
        PMLP                & $13.20$ & $40$ & $19$ & $21$ & $47.25$ \\
        PCT & $2.34$  & $23$ &  $6$ & $17$ & $12.34$ \\
        PTv3    & $13.73$ & $90$ & $28$ & $62$ & $13.97$ \\
        \bottomrule
    \end{tabular}
    \caption{CDG analysis on various 3D architectures computed purely from graph topology in $<1$ second, providing an immediate compression feasibility assessment before any search begins.}
    \label{tab:cdg-analysis}
\end{table}

Table~\ref{tab:cdg-analysis} shows the CDG analysis where the values are computed from ONNX graph topology only, no weights, no data, no execution required. Free = freely prunable Conv/Gemm nodes. Constrained = nodes locked by~\ref{itm:R1}–\ref{itm:R5}. 
All results are machine-verified through Theorem~\ref{thm:rho-f-invariant} on the ONNX files. It also works on 2D classification tasks\footnote{Results are presented in Supplementary document.} 

\subsection{H3DNAS}
\label{ssec:h3dnas}
 
The H3DNAS framework runs the same fixed steps for every model. First, the \textbf{ONNXParser} parses the \texttt{.onnx} file, performs shape inference, and computes the FLOPs and parameters. Next, the \textbf{CDG Analysis} classifies nodes and computes $\rho_f$. This is followed by the \textbf{Base Accuracy} step, which performs an optional ORT evaluation on a labeled subset. The \textbf{ArchitectureGenerator} then creates width-scaled candidates using either random or evolutionary search methods. Afterward, the \textbf{ArchitectureModulator} applies L1-guided channel pruning with BN alignment. Finally, the workflow executes the \textbf{OutputFidelityScorer} to measure the cosine similarity between the base and pruned logits without requiring labeled data, alongside the \textbf{HardwareEvaluator} which evaluates ORT latency, assesses accuracy, and performs the final constraint checks.

Figure~\ref{fig:h3dnas_pipeline} shows the H3DNAS pipeline.
Every step operates on \texttt{ONNXModel}, a thin wrapper around \texttt{onnx.ModelProto} that carries the full graph with real weights, preserving all operator attributes and shape annotations.
No intermediate JSON, YAML, or custom representation is used; the ONNX object is the sole intermediate representation throughout.

\vspace{-1em}\paragraph{Stage 1: Structural Pruning.}
The generator samples candidates by applying width multipliers $w \in W$ and prune ratios $r \in R$ to free nodes. Figure~\ref{fig:l1_guided_pruning} shows that L1 pruning selects the top-$(1-r)$ output channels by L1 norm, propagating new channel indices downstream through passthrough operators and realigning BatchNorm parameters to the same channel selection. \textbf{Output fidelity} pre-screens all candidates using $N=32$ random inputs: for each candidate, cosine similarity between the base model's logits and the pruned model's logits is computed; only the top-$K$ (default $K=15$) by output fidelity proceed to labeled ORT evaluation. Unlike SWAP-Score~\cite{peng2024swap}, output fidelity directly measures how much the pruned model's predictions diverge from the base - a more stable signal at high compression ratios where activation patterns collapse.
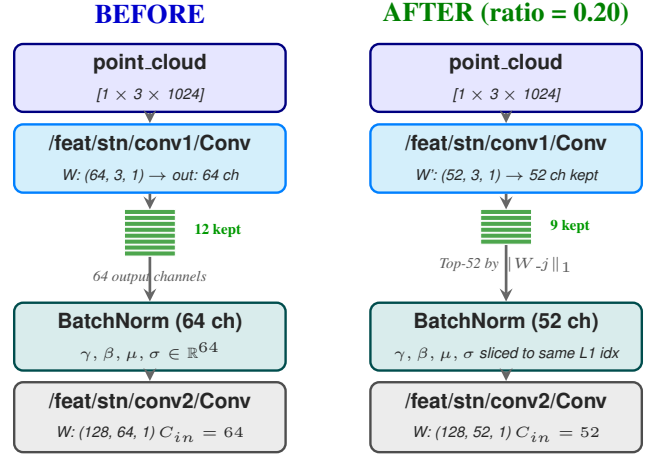
\begin{figure}[!h]
\centering
\resizebox{\columnwidth}{!}{%
\begin{tikzpicture}[
    font=\sffamily,
    >=stealth,
    inputbox/.style={
        rounded corners=3pt,
        draw=blue!50!black,
        fill=blue!10,
        line width=0.8pt,
        minimum width=3.2cm,
        minimum height=0.65cm,
        align=center,
        text=black!85
    },
    conv1box/.style={
        rounded corners=3pt,
        draw=blue!50!cyan,
        fill=cyan!15,
        line width=0.8pt,
        minimum width=3.2cm,
        minimum height=0.65cm,
        align=center,
        text=black!85
    },
    bnbox/.style={
        rounded corners=3pt,
        draw=teal!60!black,
        fill=teal!15,
        line width=0.8pt,
        minimum width=3.2cm,
        minimum height=0.65cm,
        align=center,
        text=black!85
    },
    conv2box/.style={
        rounded corners=3pt,
        draw=gray!60!black,
        fill=gray!15,
        line width=0.8pt,
        minimum width=3.2cm,
        minimum height=0.65cm,
        align=center,
        text=black!85
    },
    downarrow/.style={
        ->,
        line width=0.9pt,
        draw=black!60
    }
]


\node[font=\bfseries\small, text=blue!80!black] at (-2.1, 0.75) {BEFORE};

\node[inputbox] (b_input) at (-2.1, 0) {
    \textbf{\scriptsize point\_cloud}\\[-2pt]
    \tiny\itshape [1 $\times$ 3 $\times$ 1024]
};

\node[conv1box] (b_conv1) at (-2.1, -0.95) {
    \textbf{\scriptsize /feat/stn/conv1/Conv}\\[-2pt]
    \tiny\itshape W: (64, 3, 1) $\rightarrow$ out: 64 ch
};

\foreach \i in {0,...,7} {
    \draw[fill=green!50!black!70, draw=white, line width=0.3pt]
        (-2.4, {-1.55 - \i*0.07}) rectangle (-1.8, {-1.61 - \i*0.07});
}
\node[anchor=west, font=\bfseries\tiny, text=green!60!black] at (-1.7, -1.8) {12 kept};
\node[font=\tiny\itshape, text=black!65] at (-2.1, -2.35) {64 output channels};

\node[bnbox] (b_bn) at (-2.1, -3.05) {
    \textbf{\scriptsize BatchNorm (64 ch)}\\[-2pt]
    \tiny\itshape $\gamma, \beta, \mu, \sigma \in \mathbb{R}^{64}$
};

\node[conv2box] (b_conv2) at (-2.1, -4.0) {
    \textbf{\scriptsize /feat/stn/conv2/Conv}\\[-2pt]
    \tiny\itshape W: (128, 64, 1) $C_{in}=64$
};

\draw[downarrow] (b_input.south) -- (b_conv1.north);
\draw[downarrow] (b_conv1.south) -- (-2.1, -1.55);
\draw[downarrow] (-2.1, -2.11) -- (b_bn.north);
\draw[downarrow] (b_bn.south) -- (b_conv2.north);


\node[font=\bfseries\small, text=green!50!black] at (2.1, 0.75) {AFTER (ratio = 0.20)};

\node[inputbox] (a_input) at (2.1, 0) {
    \textbf{\scriptsize point\_cloud}\\[-2pt]
    \tiny\itshape [1 $\times$ 3 $\times$ 1024]
};

\node[conv1box] (a_conv1) at (2.1, -0.95) {
    \textbf{\scriptsize /feat/stn/conv1/Conv}\\[-2pt]
    \tiny\itshape W': (52, 3, 1) $\rightarrow$ 52 ch kept
};

\foreach \i in {0,...,5} {
    \draw[fill=green!50!black!70, draw=white, line width=0.3pt]
        (1.8, {-1.55 - \i*0.07}) rectangle (2.4, {-1.61 - \i*0.07});
}
\node[anchor=west, font=\bfseries\tiny, text=green!60!black] at (2.5, -1.75) {9 kept};
\node[font=\tiny\itshape, text=black!65] at (2.1, -2.2) {Top-52 by $\|W\_j\|_1$};

\node[bnbox] (a_bn) at (2.1, -3.05) {
    \textbf{\scriptsize BatchNorm (52 ch)}\\[-2pt]
    \tiny\itshape $\gamma, \beta, \mu, \sigma$ sliced to same L1 idx
};

\node[conv2box] (a_conv2) at (2.1, -4.0) {
    \textbf{\scriptsize /feat/stn/conv2/Conv}\\[-2pt]
    \tiny\itshape W: (128, 52, 1) $C_{in}=52$
};

\draw[downarrow] (a_input.south) -- (a_conv1.north);
\draw[downarrow] (a_conv1.south) -- (2.1, -1.55);
\draw[downarrow] (2.1, -1.97) -- (a_bn.north);
\draw[downarrow] (a_bn.south) -- (a_conv2.north);

\end{tikzpicture}%
}
\caption{PointNet ONNX graph before (left) and after (right) the L1-guided channel pruning.}
\label{fig:l1_guided_pruning}
\end{figure}

\begin{figure*}[t]
\centering
\resizebox{\textwidth}{!}{%
\begin{tikzpicture}[
    font=\sffamily,
    >=stealth,
    box/.style={
        rounded corners=3pt,
        draw=#1!60!black,
        line width=1.2pt,
        fill=#1!5,
        minimum height=5.0cm,
        minimum width=3.9cm
    },
    module/.style={
        rounded corners=2pt,
        fill=#1!25!white,
        draw=#1!50!black,
        line width=0.8pt,
        text=black!85,
        minimum width=3.5cm,
        minimum height=1.35cm,
        align=center,
        font=\bfseries
    },
    arrow/.style={
        ->,
        line width=1.5pt,
        draw=black!65
    }
]

\node[box=blue] (parserbox) at (0,-0.4) {};

\node[module=blue] (parser1) at (0,1.3)
{ONNXParser\\[3pt]\normalfont\small\itshape Load + shape inference};

\node[module=cyan] (parser2) at (0,-0.5)
{CDG Analysis\\[3pt]\normalfont\small\itshape Classify ops, compute $\rho_f$};

\node[font=\bfseries\large, text=blue!70!black] at (0,-2.35) {Parser};

\node[box=green] (generatorbox) at (4.35,-0.4) {};

\node[module=green, minimum height=1.15cm] (generator1) at (4.35,1.4)
{Arch Generator\\[2pt]\normalfont\small\itshape Width $\times$ prune ratio grid};

\node[module=green!80!teal, minimum height=1.15cm] (generator2) at (4.35,0.0)
{Search Strategy\\[2pt]\normalfont\small\itshape Evolutionary / random};

\node[module=green!60!yellow, minimum height=1.15cm] (generator3) at (4.35,-1.4)
{GhostConv (Optional)\\[2pt]\normalfont\small\itshape Structural Mutation};

\node[font=\bfseries\large, text=green!60!black] at (4.35,-2.35) {Generator};

\node[box=purple] (modulatorbox) at (8.7,-0.4) {};

\node[module=purple] (modulator1) at (8.7,1.3)
{L1-Graph Pruner\\[3pt]\normalfont\small\itshape Channel selection +\\[0pt]\normalfont\small\itshape $C_{in}$ prop.};

\node[module=purple!80!violet] (modulator2) at (8.7,-0.5)
{BN Alignment\\[3pt]\normalfont\small\itshape Scale, bias, mean, var sync};

\node[font=\bfseries\large, text=purple!70!black] at (8.7,-2.35) {Modulator};

\node[box=red] (evaluatorbox) at (13.05,-0.4) {};

\node[module=red] (evaluator1) at (13.05,1.3)
{Output Fidelity\\[3pt]\normalfont\small\itshape Cosine sim $\rightarrow$ Top-K filter};

\node[module=red!80!orange] (evaluator2) at (13.05,-0.5)
{ORT Evaluator\\[3pt]\normalfont\small\itshape Latency P50 +\\[0pt]\normalfont\small\itshape top-1 accuracy};

\node[font=\bfseries\large, text=red!70!black] at (13.05,-2.35) {Evaluator};

\node[
    rounded corners=3pt,
    draw=orange!70!black,
    line width=1.2pt,
    fill=orange!5,
    minimum height=5.0cm,
    minimum width=2.9cm
] (outputbox) at (16.85,-0.4) {};

\node[align=center, font=\bfseries\large, text=orange!80!black] at (16.85,0.8)
{Pareto\\Frontier};

\node[align=center, font=\small\itshape, text=black!70] at (16.85,-0.6)
{Compressed ONNX\\Jetson-feasible};

\node[font=\bfseries\large, text=orange!80!black] at (16.85,-2.35) {Output};

\node[
    rounded corners=4pt,
    draw=blue!50!black,
    fill=blue!15,
    line width=1pt,
    minimum width=1.3cm,
    minimum height=0.8cm,
    align=center,
    font=\bfseries\small,
    text=black!80
] (input) at (-3.2,0) {ONNX\\File};

\draw[arrow] (input.east) -- (parserbox.west |- input.east);
\draw[arrow] (parserbox.east |- input.east) -- (generatorbox.west |- input.east);
\draw[arrow] (generatorbox.east |- input.east) -- (modulatorbox.west |- input.east);
\draw[arrow] (modulatorbox.east |- input.east) -- (evaluatorbox.west |- input.east);
\draw[arrow] (evaluatorbox.east |- input.east) -- (outputbox.west |- input.east);

\end{tikzpicture}%
}
\caption{
Overview of the proposed H3DNAS pipeline.
The system parses an ONNX model, analyzes channel dependencies,
generates candidate architectures using search strategies and optional GhostConv mutations,
performs structured graph pruning, and evaluates accuracy and hardware-related metrics.
The resulting candidates form a Pareto frontier.
}
\label{fig:h3dnas_pipeline}
\end{figure*}
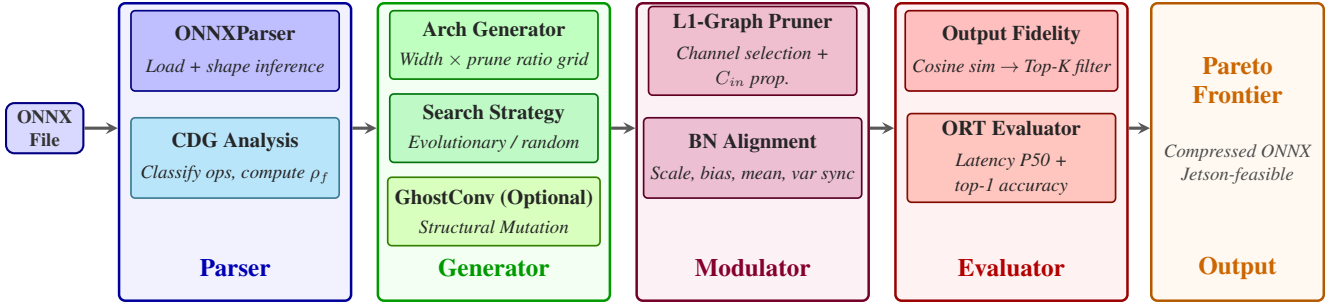

\paragraph{Stage 2: GhostConv Structural Mutation.}
The top-$K_2$ Stage 1 Pareto candidates receive GhostConv mutations. For each eligible free Conv node ($\text{groups}=1$, $C_{out} \geq 16$), the mutator replaces $\mathrm{Conv}(C_{in}, C_{out}, k)$ with: Primary $\mathrm{Conv}(C_{in}, C_{out}\text{//}2, k)$ + Ghost $\mathrm{DWConv}(C_{out}\text{//}2, s, \mathrm{groups}=C_{out}\text{//}2)$ + Concat. Ghost weights are initialized as near-identity ($\text{centre}=1.0$, $\text{rest}=0.0$), preserving the output distribution at zero-shot. Stage 2 candidates are ranked by output fidelity only; GhostConv models with near-identity initialization have meaningful fidelity to the base model's output distribution, making output fidelity a more reliable proxy than accuracy for zero-shot ghost models.
\begin{figure}[!h]
\centering
\resizebox{\columnwidth}{!}{%
\begin{tikzpicture}[
    font=\sffamily,
    >=stealth,
    inputbox/.style={
        rounded corners=3pt,
        draw=blue!50!black,
        fill=blue!10,
        line width=0.8pt,
        minimum width=3.3cm,
        minimum height=0.65cm,
        align=center,
        text=black!85
    },
    purplebox/.style={
        rounded corners=3pt,
        draw=violet!60!black,
        fill=violet!15,
        line width=0.8pt,
        minimum width=3.3cm,
        minimum height=0.65cm,
        align=center,
        text=black!85
    },
    tealbox/.style={
        rounded corners=3pt,
        draw=teal!60!black,
        fill=teal!15,
        line width=0.8pt,
        minimum width=3.3cm,
        minimum height=0.65cm,
        align=center,
        text=black!85
    },
    graybox/.style={
        rounded corners=3pt,
        draw=gray!60!black,
        fill=gray!15,
        line width=0.8pt,
        minimum width=3.3cm,
        minimum height=0.65cm,
        align=center,
        text=black!85
    },
    greenbox/.style={
        rounded corners=3pt,
        draw=green!60!black,
        fill=green!15,
        line width=0.8pt,
        minimum width=3.3cm,
        minimum height=0.65cm,
        align=center,
        text=black!85
    },
    annbox/.style={
        rounded corners=2pt,
        draw=violet!60!black,
        fill=violet!5,
        line width=0.4pt,
        inner sep=1.2pt,
        align=center,
        font=\tiny\itshape
    },
    flopbox/.style={
        rounded corners=2pt,
        draw=gray!50,
        fill=white,
        line width=0.4pt,
        inner sep=1.5pt,
        align=center,
        font=\tiny\itshape
    },
    downarrow/.style={
        ->,
        line width=0.9pt,
        draw=black!60
    }
]


\node[font=\bfseries\small, text=blue!80!black] at (-2.25, 0.75) {BEFORE};

\node[inputbox] (b_input) at (-2.25, 0) {
    \textbf{\scriptsize /feat/stn/Relu\_output\_0}\\[-2pt]
    \tiny\itshape [1 $\times$ 64 $\times$ 1024]
};

\node[purplebox] (b_conv) at (-2.25, -0.95) {
    \textbf{\scriptsize /feat/stn/conv2/Conv}\\[-2pt]
    \tiny\itshape W: (128, 64, 1) $\rightarrow$ 128 channels
};

\node[tealbox] (b_relu) at (-2.25, -1.9) {
    \textbf{\scriptsize Relu}\\[-2pt]
    \tiny\itshape /feat/stn/Relu\_1\_output\_0
};

\node[graybox] (b_conv3) at (-2.25, -2.85) {
    \textbf{\scriptsize /feat/stn/conv3/Conv}\\[-2pt]
    \tiny\itshape W: (1024, 128, 1)
};

\node[flopbox] at (-2.25, -3.6) {
    FLOPs: 2 $\times$ 1024 $\times$ 128 $\times$ 64 = 16.7M
};

\draw[downarrow] (b_input.south) -- (b_conv.north);
\draw[downarrow] (b_conv.south) -- (b_relu.north);
\draw[downarrow] (b_relu.south) -- (b_conv3.north);


\node[font=\bfseries\small, text=green!50!black] at (2.45, 0.75) {AFTER};

\node[inputbox] (a_input) at (2.45, 0) {
    \textbf{\scriptsize /feat/stn/Relu\_output\_0}\\[-2pt]
    \tiny\itshape [1 $\times$ 64 $\times$ 1024] (unchanged)
};

\node[purplebox] (a_primary) at (2.45, -0.95) {
    \textbf{\scriptsize /feat/stn/conv2/Conv\_ghost0\_primary}\\[-2pt]
    \tiny\itshape W: (64, 64, 1) $\rightarrow$ 64 channels
};

\node[purplebox] (a_ghost) at (2.45, -1.9) {
    \textbf{\scriptsize /feat/stn/conv2/Conv\_ghost0\_ghost}\\[-2pt]
    \tiny\itshape W: (64, 1, 1) groups=64 (depthwise)
};
\node[annbox, anchor=west] at (4.15, -2.05) {init $\approx$\\near-I};

\node[greenbox] (a_concat) at (2.45, -2.85) {
    \textbf{\scriptsize /feat/stn/conv2/Conv\_ghost0\_concat}\\[-2pt]
    \tiny\itshape [64 $|$ 64] $\rightarrow$ 128 channels (same Cout)
};

\node[tealbox] (a_relu) at (2.45, -3.8) {
    \textbf{\scriptsize Relu}\\[-2pt]
    \tiny\itshape /feat/stn/Relu\_1\_output\_0 (unchanged)
};

\node[font=\scriptsize\itshape, text=black!65] at (2.45, -4.5) {
    $\rightarrow$ /feat/stn/conv3/Conv (downstream unchanged)
};

\draw[downarrow] (a_input.south) -- (a_primary.north);
\draw[downarrow] (a_primary.south) -- (a_ghost.north);
\draw[downarrow] (a_ghost.south) -- (a_concat.north);
\draw[downarrow] (a_concat.south) -- (a_relu.north);
\draw[downarrow] (a_relu.south) -- (2.45, -4.35);




\end{tikzpicture}%
}
\caption{Impact of GhostConv mutation replacing a single convolution layer with a cheaper two branch equivalent on the original PointNet ONNX Graph.}
\label{fig:ghostconv_mutation}
\end{figure}
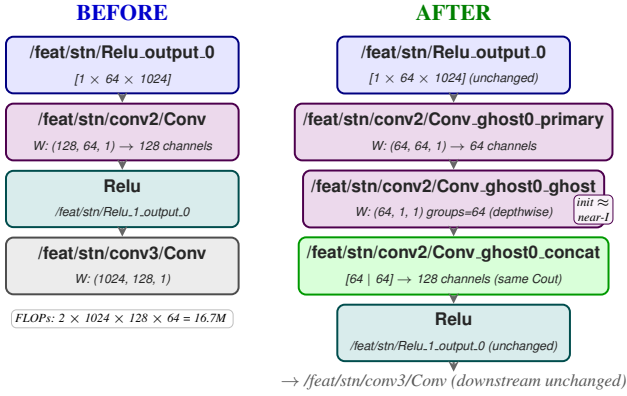

\vspace{-1em}\paragraph{Unified Pareto.}
Stage 1 and Stage 2 candidates are merged into a single Pareto frontier across accuracy and FLOPs. Hardware constraint filtering (FLOPs, params, latency, model size) is applied to identify feasible candidates.

\subsection{Hardware Constraint Integration}
\label{sec:hardware-constraints}
 
\texttt{HardwareConstraints} encodes deployment budgets as hard limits on parameter count, FLOPs, model size, and ORT latency. Built-in presets target Jetson Orin Nano 8GB, any other device is supported by specifying custom limits. 
Candidates violating any constraint are classified as infeasible and excluded from ranking. The search therefore implicitly optimises over the feasible Pareto frontier.

\section{Experiments}
\label{sec:experiments}

\subsection{Experimental Setup}
\label{sec:experimental-setup}

\paragraph{Models and Datasets.}
We evaluate on PointNet~\cite{qi2017pointnet}, PointNet++ SSG~\cite{qi2017pointnet2} and PointMLP~\cite{ma2022pointmlp} for 3D point cloud classification on ModelNet40~\cite{wu2015modelnet} ($9{,}843$ train / $2{,}468$ test samples, $40$ classes, $1{,}02$4 points per cloud), the standard benchmark used in all prior 3D compression work. All base models are loaded from exported ONNX files, no PyTorch source code is used at any stage of search, compression, or evaluation.

\vspace{-0.9em}\paragraph{Target Hardware.}
NVIDIA Jetson Orin Nano 8GB~\cite{nvidia2024orinnano}: 5 TFLOPS FP32, 8\,GB LPDDR5 @ 68\,GB/s, $7-15$W TDP. Hard deployment budget: $\leq$4B FLOPs, $\leq$20M parameters, $\leq$50\,MB model size, $\leq$50\,ms latency. All latency measurements use OnnxRuntime CPU execution ($\text{batch}=1$, $100$ timed runs, $\text{P}50$ median). Jetson feasibility is checked against these constraints automatically during search.

\vspace{-0.9em}\paragraph{NAS Configuration.}
Two-stage H3DNAS: Stage 1 searches prune ratios $[0.05-0.50]$ $\times$ width multipliers $[0.5-1.0]$, strategy$=$evolutionary, $\text{seed}=105$. Output fidelity ($n=32$ random inputs, cosine similarity of base vs.\ pruned logits) pre-screens all candidates; top-15 proceed to ORT accuracy evaluation on the full test set. Stage 1 Pareto front feeds Stage 2 GhostConv~\cite{han2020ghostnet} mutation, ranked by output fidelity only.

\vspace{-0.9em}\paragraph{Fine-tuning.}
Adam, $\text{lr}=5\times10^{-4}$, $\text{weight decay}=10^{-4}$, 15 epochs, cosine-annealing LR schedule. Feature-transform regulariser $\lambda=0.001$ for PointNet. All fine-tuning uses \texttt{onnx2torch} to reconstruct a trainable model directly from the pruned ONNX --- no original architecture definition required.

\vspace{-0.9em}\paragraph{Baselines.}
All baselines use an identical fine-tuning setup and the same H3DNAS graph surgery ($C_{in}$ propagation, BN alignment), only the channel importance criterion differs:

\begin{itemize}
    \item \textbf{Uniform L1}~\cite{li2017pruning}: Same L1 prune ratio applied uniformly to all free nodes
    \item \textbf{HRank}~\cite{Lin_2020_hrank}: Channels ranked by average feature map rank from $64$ random inputs; base criterion used in CP\textsuperscript{3}~\cite{huang2023cp3}
    \item \textbf{Screening F-stat}~\cite{wang2025screening}: Channels ranked by ANOVA F-statistic over pseudo-labeled activations
\end{itemize}

\subsection{CDG Theorem Validation}
\label{sec:cdg-validation}

Table~\ref{tab:cdg-theorem-validation} reports the CDG analysis for all tested models, computed from ONNX graph topology in $\mathcal{O}(|V|+|E|)$ without requiring trained weights, calibration data, or model execution.

\begin{table}[!h]
    \centering
    \setlength{\tabcolsep}{5pt}
    \begin{tabular}{lccccc}
        \toprule
        \multirow{2}{*}{Model} & $\rho_f$ & \multirow{2}{*}{Achieved} & Ceiling & \multirow{2}{*}{Verdict} \\
        & (verified)    &    & Gap    &        \\
        \midrule
        PN                   & $58.13\%$ & $65.5\%$ & ---     & \checkmark\ \\
        PN++             & $61.43\%$ & $43.2\%$ & $18.2$pp  & \checkmark\ \\
        PMLP                   & $47.25\%$ & $49.1\%$ & ---     & \checkmark\  \\
        PCT    & $12.34\%$ & $6.7\%$  & $5.64$pp  & \ding{55}\ \\
        PTv3       & $13.97\%$ & $11.1\%$ & $2.87$pp     & \ding{55}\ \\
        \bottomrule
    \end{tabular}
    \caption{CDG theorem validation across model families. $\rho_f$ = free parameter fraction (topological invariant), machine-verified via Theorem~\ref{thm:rho-f-invariant} on ONNX files. Achieved = actual parameter reduction delivered by H3DNAS. Ceiling gap = $\rho_f -$ achieved.}
    \label{tab:cdg-theorem-validation}
\end{table}

The Free Parameter Fraction($\rho_f$) is verified as a topological invariant: the Free/Constrained node split (12/6 for PointNet~\cite{qi2017pointnet}, 9/3 for PointNet++~\cite{qi2017pointnet2}, 19/21 for PointMLP~\cite{ma2022pointmlp}) is identical between base and H3DNAS ONNX variants, changing only when GhostConv mutations introduce new channel-in couplings (small $\Delta\rho_f \approx 0.03$--$0.05$). Point Transformer V3~\cite{wu2024pointtransv3} ($\rho_f = 13.97\%$) confirm the theorem's utility as a pre-search decision tool: CDG correctly identifies architectures where channel pruning cannot meet the Jetson budget before any experiment is run\footnote{Full regime analysis in the Supplementary Document}.

\vspace{-0.5em}
\paragraph{Topological Invariant Preservation After Compression.}
Since topological invariance should remain stable after pruning. Table~\ref{tab:rho-f-preservation} verifies this directly in the final paper models: the Free/Constrained node partition is preserved exactly and $\rho_f$ changes only marginally ($\Delta \approx 0.03$--$0.05$) when GhostConv mutations introduce new channel-in couplings in the H3DNAS variant.

\begin{table}[!h]
    \centering
    \resizebox{\columnwidth}{!}{
    \setlength{\tabcolsep}{5pt}
    \begin{tabular}{l|ccc|lc}
        \toprule
        \multirow{2}{*}{Model} & \multicolumn{3}{c|}{$\rho_f$} & Free/ & \multirow{2}{*}{Same} \\\cline{2-4}
        & Base & H3DNAS & $\Delta$ & Constrained \\\midrule
        PN          & $0.5813$ & $0.5468$ & $-0.0345$ & $12$ / $6$ $\rightarrow$ $12$ / $6$ & \checkmark \\
        PN++     & $0.6143$ & $0.5832$ & $-0.0311$ & $9$ / $3$ $\rightarrow$ $9$ / $3$ & \checkmark \\
        PMLP            & $0.4725$ & $0.4191$ & $-0.0534$ & $19$ / $21$ $\rightarrow$ $19$ / $21$ & \checkmark \\
        PCT            & $0.1234$ & $0.1234$ & $0.00$ & $6$ / $17$ $\rightarrow$ $6$ / $17$ & \checkmark \\
        PTv3           & $0.0705$ & $0.0705$ & $0.00$ & $15$ / $75$ $\rightarrow$ $15$ / $75$ & \checkmark \\
        \bottomrule
    \end{tabular}
    }
    \caption{Free parameter reduction($\rho_f$)preservation --- base vs.\ H3DNAS ONNX models. Same Free/Constrained partition confirms Theorem~\ref{thm:rho-f-invariant} (topology-only invariant). $\Delta\rho_f$ from GhostConv stage 2 coupling.}
    \label{tab:rho-f-preservation}
\end{table}

Because the node partition remains identical across all three cases, this confirms that this metric depends solely on graph topology. The slight reduction observed in H3DNAS stems directly from these Stage~2 mutations, which replace a free Conv with a primary Conv + depthwise ghost branch, introducing new channel coupling that promotes some free nodes to constrained. This is expected and consistent with Theorem~\ref{thm:rho-f-invariant}.


\begin{table*}[!tb]
    \centering
    \small
    \setlength{\tabcolsep}{5pt}
    \begin{tabular}{llrrrrrrr}
        \toprule
        Model & Variant & Accuracy & $\Delta$ (pp) & Params & Param$\downarrow$ & FLOPs$\downarrow$ & Size & Latency (P50) \\
        \midrule
        \multirow{2}{*}{PointNet}       & Base    & $90.32\%$          & ---               & $3.46$M  & ---             & ---    & 13.32MB & 11.98ms \\
                                          & H3DNAS  & \textbf{$90.28\%$} & \textbf{$-0.04$}  & $1.19$M  & \textbf{$65.5\%$} & $63.4\%$ & 4.65MB  & 6.02ms (\textbf{1.99$\times$}) \\
        \cmidrule{1-9}
        \multirow{2}{*}{PointNet++ SSG} & Base    & $91.90\%$          & ---               & $1.47$M  & ---             & ---    & 5.70MB  & 37.14ms \\
                                          & H3DNAS  & \textbf{$91.98\%$} & \textbf{$+0.08$ \checkmark} & $0.83$M  & \textbf{$43.2\%$} & $43.1\%$ & 3.27MB  & 28.86ms (\textbf{1.29$\times$}) \\
        \cmidrule{1-9}
        \multirow{2}{*}{PointMLP}        & Base    & $93.40\%$          & ---               & 13.20M & ---             & ---    & 50.86MB & 347.0ms \\
                                          & H3DNAS  & \textbf{93.11\%} & \textbf{$-0.28$}  & 6.72M  & \textbf{$49.1\%$} & 48.8\% & 26.07MB & 207.8ms (\textbf{1.67$\times$}) \\
        \bottomrule
    \end{tabular}
    \caption{H3DNAS results on three 3D point cloud architectures (ModelNet40, $2{,}468$ test samples, ORT-verified). All accuracy values are post-fine-tuning top-1. Latency column reports P50 base/H3DNAS, with speedup (base/H3DNAS) shown in parentheses on the H3DNAS row. Strategy: evolutionary NAS with output fidelity scoring.}
    \label{tab:main-results}
\end{table*}
 
Table~\ref{tab:main-results} highlights four key findings. First, \textbf{PointNet++ SSG achieves super-baseline accuracy} ($+0.08$\,pp) at $43.2\%$ parameter reduction and $1.29\times$ speedup which is consistent with CP$^3$~\cite{huang2023cp3} ($+0.15$\,pp at $43\%$ with PyTorch source code) but achieved from the ONNX file alone. Second, \textbf{PointNet achieves near-lossless compression} ($-0.04$\,pp) at $65.5\%$ parameter reduction and 1.99$\times$ speedup with the smallest accuracy gap at this compression level. Third, \textbf{PointMLP achieves $93.11\%$} ($-0.28$,pp) with a $49.1\%$ reduction in parameters, outperforming all prior PointMLP compression methods including T3DNet~\cite{yang2024t3dnet} distillation ($\sim 91.0\%$) and HLS4PC~\cite{pal2025hls4pc} input pruning ($91.69\%$). Fourth, all three results are obtained from standalone ONNX files, requiring neither the original source code, the model architecture class, nor the training framework APIs.

\begin{table}[!h]
    \centering
    \resizebox{\columnwidth}{!}{
    \setlength{\tabcolsep}{5pt}
    \begin{tabular}{llrrrr}
        \toprule
        Model & Priority & NAS acc & $\Delta$ (pp) & Param$\downarrow$ & Speedup \\
        \midrule
        \multirow{3}{*}{PN}   & \textbf{Accuracy}    & \textbf{$90.28\%$} & \textbf{$-0.04$} & \textbf{$65.5\%$} & \textbf{1.99$\times$} \\
                                      & Balanced                      & $90.07\%$          & $-0.24$          & $32.3\%$          & 1.53$\times$ \\
                                      & Compression           & $89.71\%$          & $-0.61$          & $79.4\%$          & 2.32$\times$ \\
        \cmidrule{1-6}
        \multirow{3}{*}{PN++} & \textbf{Accuracy}    & \textbf{91.98\%} & \textbf{$+0.08$} & \textbf{43.2\%} & \textbf{2.82$\times$} \\
                                      & Balanced                      & $91.77\%$          & $-0.12$          & $33.8\%$          & 2.67$\times$ \\
                                      & Compression           & $90.32\%$          & $-1.58$          & $80.7\%$          & 3.38$\times$ \\
        \cmidrule{1-6}
        \multirow{3}{*}{PMLP}    & \textbf{Accuracy}    & \textbf{$93.11\%$} & \textbf{$-0.28$} & \textbf{$49.1\%$} & \textbf{1.44$\times$} \\
                                      & Balanced                      & $93.03\%$          & $-0.36$          & $17.8\%$          & 1.14$\times$ \\
                                      & Compression           & $92.83\%$          & $-0.57$          & $18.6\%$          & 1.20$\times$ \\
        \bottomrule
    \end{tabular}
    }
    \caption{H3DNAS Pareto frontier. Full compression spectrum from one NAS run (ModelNet40). Values in bold indicate the best-performing results. All results ORT-verified using \texttt{CPUExecutionProvider}.}
    \label{tab:pareto-spectrum}
\end{table}

The Pareto frontier shown in Table~\ref{tab:pareto-spectrum} reveals architecture-specific compression behavior validated by CDG analysis: PointNet++ has the steepest accuracy cliff at high compression ($-1.58$\,pp at $80.7\%$), consistent with its 7 constrained nodes; PointNet's flat curve reflects $11$ free nodes and $\rho_f = 57.96\%$; PointMLP's tight cluster reflects many small-gain-free layers. No fixed-ratio baseline delivers this cross-architecture analysis from a single run.

\subsection{Jetson Edge Deployment}
\label{sec:jetson-deployment}
Table~\ref{tab:jetson-constraints} shows that the base PointNet violates three of four Jetson constraints. H3DNAS automatically discovers a model satisfying all four, with substantial headroom on every metric. The latency constraint has the most headroom ($+85.2\%$) because the Jetson ORT CPU latency of 7.4\,ms is well within the 50\,ms budget which is consistent with real-time point cloud inference requirements. The hardware constraint check runs inside the search loop, so no manual tuning of compression targets is needed.

\begin{table}[!h]
    \centering
    \resizebox{\columnwidth}{!}{
    \setlength{\tabcolsep}{5pt}
    \begin{tabular}{lrrrr}
        \toprule
        Metric & Budget & Base & H3DNAS & Headroom \\
        \midrule
        Parameters   & $2\text{M}$ & $3{,}451{,}859$ \ding{55} & $\mathbf{1{,}463{,}271 \mathrel{\checkmark}}$ & $+26.8\%$ \\
        FLOPs         & $500\text{M}$           & 870M \ding{55}          & \textbf{386M \checkmark}          & $+22.8\%$ \\
        Model size    & $50\text{MB}$           & 13.3 MB \checkmark      & \textbf{5.63 MB \checkmark}       & $+88.7\%$ \\
        Latency P50   & $50\text{ms}$           & 12.7 ms \checkmark      & \textbf{7.4 ms \checkmark}        & $+85.2\%$ \\
        Accuracy      & ---                    & $90.32\%$                 & $90.28\%$                          & $-0.04$ pp \\
        \bottomrule
    \end{tabular}
    }
    \caption{NVIDIA Jetson Orin Nano 8GB constraint satisfaction for PointNet (ModelNet40). Headroom = percentage under budget (positive = within budget). Budget = 10\% of Jetson Orin Nano 8GB device capabilities.}
    \label{tab:jetson-constraints}
\end{table}

\subsection{Ablation Study}
\label{sec:ablation}

On performing ablation for Stage 1 vs.\ Stage 2, various architecture-dependent findings emerge as shown in Table~\ref{tab:ablation-stages}. (i) PointNet benefits from GhostConv Stage 2: at identical 32.3\% compression, Stage 2 achieves 1.80$\times$ speedup vs.\ Stage 1's 1.46$\times$ which is a $+0.34\times$ gain from GhostConv's cheaper depthwise operations, at a cost of only $-0.08$pp accuracy. (ii) PointNet has 12 eligible free Conv nodes with $C_{out} \geq 16$, giving GhostConv ample mutation targets. (iii) PointNet++ exhibits marginal gains from Stage 2 ($2.64\times$ vs.\ $2.67\times$, with $-0.12$,pp for both), as the constrained topology of its set abstraction layers yields fewer candidate operations for GhostConv substitution. (iv) PointMLP is hurt by Stage 2: accuracy drops by $-0.29$pp more and speedup decreases. (v) PointMLP's residual coupling (21 of 40 nodes constrained by~\ref{itm:R2} Add) means GhostConv mutations introduce new Concat channel dependencies that interact with existing constraints, reducing compression efficiency.

\begin{table}[!h]
    \centering
    \resizebox{\columnwidth}{!}{
    \setlength{\tabcolsep}{5pt}
    \begin{tabular}{llrrrrc}
        \toprule
        Model & Variant & Acc & $\Delta$ (pp) & Param$\downarrow$ & Speedup & GC \\
        \midrule
        \multirow{3}{*}{PN}    & Base    & $90.32\%$          & ---              & ---             & ---                   & --- \\
                  & Stage 1 & $90.07\%$          & $-0.24$          & $32.3\%$          & 1.46$\times$          & \ding{55} \\
                  & Stage 2 & \textbf{$89.99\%$} & \textbf{$-0.32$} & \textbf{$32.3\%$} & \textbf{1.80$\times$} & \checkmark \\\midrule
        \multirow{3}{*}{PN++}  & Base    & $91.90\%$          & ---              & ---             & ---                   & --- \\
                  & Stage 1 & \textbf{$91.77\%$} & \textbf{$-0.12$} & $33.8\%$          & \textbf{2.67$\times$} & \ding{55} \\
                  & Stage 2 & \textbf{$91.77\%$} & \textbf{$-0.12$} & \textbf{$34.1\%$} & 2.64$\times$          & \checkmark \\\midrule
        \multirow{3}{*}{PMLP}     & Base    & $93.40\%$          & ---              & ---             & ---                   & --- \\
                  & Stage 1 & \textbf{$93.11\%$} & \textbf{$-0.28$} & \textbf{$26.8\%$} & \textbf{1.24$\times$} & \ding{55} \\
                  & Stage 2 & $92.83\%$          & $-0.57$          & $18.6\%$          & 1.20$\times$          & \checkmark \\
        \bottomrule
    \end{tabular}
    }
    \caption{H3DNAS Stage 1 vs.\ Stage 2 (GC: GhostConv mutation) operating points across all three architectures on ModelNet40 dataset. Bold marks the better-performing variant per model on the accuracy/speedup trade-off.}
    \label{tab:ablation-stages}
\end{table}
This architecture-dependent behaviour is predicted by the CDG analysis: GhostConv Stage 2 yields the greatest benefit when $\rho_f$ is high and many free nodes have $C_{out} \geq 16$. For architectures with dense residual constraints (PointMLP:~\ref{itm:R2} $\times 20$), Stage 1 channel pruning alone is recommended.

\subsection{Comparison with Existing Methods}
Unlike existing approaches, H3DNAS operates entirely without source code or framework definitions. On PointNet, it establishes the first structured compression result on ModelNet40, surpassing the original published accuracy ($89.2\%$) at $65.5\%$ parameter reduction\footnote{Detailed Analysis in Supplementary Document\label{foot:published_methods}}. On PointNet++, it matches CP\textsuperscript{3}'s compression ratio ($43\%$) from the ONNX file alone without PyTorch source code\footref{foot:published_methods}. On PointMLP, H3DNAS achieves $93.11\%$ ($-0.28$\,pp) at $49.1\%$ parameter reduction, making this the \textbf{first work to achieve direct compression on the trained PointMLP architectures}. The two prior entries in Table~\ref{tab:comparison-published} are methodologically distinct: T3DNet~\cite{yang2024t3dnet} achieves its $98\%$ by designing and training a new tiny model from scratch using the original as a teacher (knowledge distillation not weight compression); HLS4PC~\cite{pal2025hls4pc} reduces the number of input points fed to the model, not the model's weights or channels. Neither method modifies or compresses the existing trained PointMLP - H3DNAS is the only work to do so, achieving best-in-class accuracy retention ($-0.28$\,pp vs.\ $-1.45$\,pp and $-1.91$\,pp) directly from the ONNX file without source code.

\begin{table}[!h]
    \centering
    \resizebox{\columnwidth}{!}{
    \setlength{\tabcolsep}{5pt}
    \begin{tabular}{llrrrrc}
        \toprule
        \multirow{2}{*}{M} & \multirow{2}{*}{Method} & \multicolumn{2}{c}{Accuracy} & \multirow{2}{*}{$\Delta$ (pp)} & \multirow{2}{*}{Param$\downarrow$} & Source \\
        &  & Base & Comp. & & & Free \\
        \midrule
        \multirow{2}{*}{\rotatebox{90}{PN}}   & LTH Unstruct~\cite{biswas_2024_lth} & 87.50\%              & 88.20\%           & $+0.7$           & 60\%\textsuperscript{\ddag} & \ding{55} \\
                                      & \textbf{H3DNAS (ours)}              & \textbf{90.32\%}    & \textbf{90.28\%} & \textbf{$-0.04$} & \textbf{65.5\%}             & \textbf{\checkmark} \\
        \midrule
        \multirow{3}{*}{\rotatebox{90}{PN++}} & CP\textsuperscript{3}+HRank~\cite{huang2023cp3} & 92.80\%\textsuperscript{\dag} & 92.95\%          & $+0.15$          & 43\%             & \ding{55} \\
                                      & CP\textsuperscript{3}+ResRep~\cite{huang2023cp3} & 92.80\%\textsuperscript{\dag} & \textbf{93.27\%} & \textbf{$+0.47$} & 44\%             & \ding{55} \\
                                      & \textbf{H3DNAS (ours)}                              & \textbf{91.90\%}              & \textbf{91.98\%} & \textbf{$+0.08$} & \textbf{43.2\%}  & \textbf{\checkmark} \\
        \midrule
        \multirow{3}{*}{\rotatebox{90}{PMLP}}    & T3DNet~\cite{yang2024t3dnet}   & 92.45\%          & $\sim$91.0\%     & $-1.45$          & 98\%\textsuperscript{\S} & \ding{55} \\
                                      & HLS4PC~\cite{pal2025hls4pc}   & 93.60\%          & 91.69\%          & $-1.91$          & ---\textsuperscript{*}   & \ding{55} \\
                                      & \textbf{H3DNAS (ours)}      & \textbf{93.40\%} & \textbf{93.11\%} & \textbf{$-0.28$} & \textbf{49.1\%}          & \textbf{\checkmark} \\
        \bottomrule
    \end{tabular}
    }
    \caption{H3DNAS vs.\ prior compression methods on ModelNet40. We use the uncompressed baseline accuracy reported in each respective work. \textsuperscript{\ddag}Unstructured weight masking --- no FLOPs/size reduction on standard hardware. \textsuperscript{\dag}CP\textsuperscript{3} uses improved OpenPoints baseline (92.80\% vs.\ standard 91.9\%). \textsuperscript{*}Estimated. \textsuperscript{\S}T3DNet achieves 98\% by training a pre-designed tiny model from scratch via Knowledge Distillation, not by compressing the original model.}
    \label{tab:comparison-published}
\end{table}

\section{Discussion}
\label{sec:discussion}

\paragraph{Framework-agnostic operation is the key enabler.}
Every prior compression tool is tied to a specific framework - PyTorch \texttt{named\_modules}, TensorFlow layer APIs, or custom training loops.
H3DNAS is the first to derive a valid compression search space purely from the ONNX graph, enabling compression of vendor-supplied models, cross-framework checkpoints, and any ONNX binary regardless of its original training framework.
Theorem~\ref{thm:rho-f-invariant} guarantees that the framework will never produce an invalid compressed graph, and no framework API is invoked on the model at any point during the search.

\vspace{-1em}
\paragraph{The theorem has practical consequences beyond compression.}
The free parameter fraction (${\rho}_{f}$) gives an immediate upper bound on the achievable reduction of FLOPs.
When ${\rho}_{f} < 10\%$ (as for EfficientNet-B0~\cite{tan2019efficientnet}), channel pruning is unlikely to yield meaningful compression, yet the framework correctly identifies this and guides the practitioner towards quantization instead.
This automatic feasibility assessment has no equivalent in prior work, which would simply attempt compression and either fail silently or produce invalid graphs.

\vspace{-1em}
\paragraph{Architecture family insights.}
Our experiments reveal a compression ceiling taxonomy: (i) simple CNNs with no residual connections are nearly fully compressible (${\rho}_{f} \approx 99\%$); (ii) residual architectures such as ResNet50~\cite{he2016resnet} and PointNet have moderate freedom ($28-58\%$); and (iii) mobile-optimized architectures with depthwise convolutions or SE attention are severely constrained ($4-12\%$).
This taxonomy is produced automatically by Theorem~\ref{thm:rho-f-invariant} with zero model-specific knowledge.

\vspace{-1em}
\paragraph{Limitations.}
The current implementation evaluates latency on the CPU ORT rather than on the Jetson GPU directly.
While CPU latency ranking is correlated with GPU latency, on-device profiling would provide stronger hardware claims.
H3DNAS-Full currently supports Level 1 mutations (activation and pooling substitution at the \texttt{onnx.ModelProto} level) and evaluates graph-mutated variants zero-shot; Level 2 structural mutations (e.g.\ depthwise decomposition) require weight re-initialization and are left for future work.
Gradient-based architecture optimization (DARTS-style) inherently requires differentiable parameters and remains outside the scope of framework-agnostic ONNX compression.

\section{Conclusion \& Future Work}
\label{sec:conclusion}

We presented \textbf{H3DNAS}, the first hardware-aware NAS framework for 3D point cloud models operating entirely from ONNX computational graphs without source code. The Channel Dependency Graph (CDG) provides a formal, provable compression ceiling ($p_f$) computable in $\mathcal{O}(|V|+|E|)$ and validated empirically across five architectures. Using ModelNet40 dataset, H3DNAS establishes the first structured compression results for point cloud architectures: PointNet achieves near-lossless compression ($-0.04$ pp) at $65.5\%$ parameter reduction with $1.99\times$ speedup; PointNet++ SSG surpasses its uncompressed baseline ($+0.08$ pp) at $43.2\%$ reduction with $1.29\times$ speedup; PointMLP achieves $93.11\%$ ($-0.28$ pp) at $49.1\%$ reduction outperforming all prior PointMLP compression methods. Under the Jetson Orin Nano 8GB constraints, H3DNAS delivers a $1.67\times$ CPU speedup with a $57.6\%$ parameter reduction, while automatically satisfying all hardware constraints. A single NAS run delivers the full Pareto frontier across compression levels, no fixed-ratio baseline can match this \footref{foot:published_methods}. All results are obtained from ONNX files alone, without any source code, architecture definition, or training framework API. All latency values are measured on the Jetson Orin Nano 8GB via ORT CPU execution and correlation with other edge devices are left for future work. Fine-tuning via onnx2torch succeeds for all tested models but may require manual intervention for operators not in onnx2torch's supported set.

\newpage

\maketitlesupplementary

\tableofcontents


\section{CDG Operator Classification and Constraint Rules}
\label{sec:supp_cdg_ops}
 
Table~\ref{tab:s1a} is the executable proof of Lemma~1: every operator in $\Omega$ (34 entries) is assigned to exactly one class by exhaustive case analysis. Implemented as \texttt{\_OP\_CLASS} in \texttt{h3dnas/parser/channel\_dependency\_graph.py}. Table~\ref{tab:s1b} then formalises the five locking rules R1--R5 derived from that classification, listing the ONNX contract each rule enforces and the architectures where it fires.
 
\begin{table}[h]
\centering
\caption{Complete ONNX operator set $\Omega$ (34 operators) with CDG classification.}
\label{tab:s1a}
\resizebox{\columnwidth}{!}{
\begin{tabular}{ll}
    \toprule
    Operator & Justification \\
    \midrule
    \multicolumn{2}{l}{\textit{Channel-Generating (CG)}} \\
    Conv              & $C_{out} = W.\text{shape}[0]$; independent of input channels \\
    Gemm              & $C_{out} = W.\text{shape}[0\,\text{or}\,1]$ per \texttt{transB} flag \\
    \midrule
    \multicolumn{2}{l}{\textit{Channel-Transparent (CT)}} \\
    ReLU              & $f\colon\mathbb{R}^C\!\to\!\mathbb{R}^C$, elementwise \\
    LeakyReLU         & $f\colon\mathbb{R}^C\!\to\!\mathbb{R}^C$, elementwise, $\alpha\in\mathbb{R}$ \\
    ELU               & $f\colon\mathbb{R}^C\!\to\!\mathbb{R}^C$, elementwise \\
    Sigmoid           & $f\colon\mathbb{R}^C\!\to\!\mathbb{R}^C$, elementwise \\
    Tanh              & $f\colon\mathbb{R}^C\!\to\!\mathbb{R}^C$, elementwise \\
    Clip              & $f\colon\mathbb{R}^C\!\to\!\mathbb{R}^C$, elementwise \\
    GELU              & $f\colon\mathbb{R}^C\!\to\!\mathbb{R}^C$, elementwise (opset$\ge$20) \\
    HardSwish         & $f\colon\mathbb{R}^C\!\to\!\mathbb{R}^C$, elementwise \\
    Mish              & $f\colon\mathbb{R}^C\!\to\!\mathbb{R}^C$, elementwise \\
    BatchNormalization & normalises per-channel; $C_{out}=C_{in}$ \\
    Dropout           & stochastic zeroing; $C_{out}=C_{in}$ \\
    Identity          & passthrough \\
    Transpose         & permutes axes; channel identity preserved \\
    MaxPool           & pools spatial dims; $C_{out}=C_{in}$ \\
    AveragePool       & pools spatial dims; $C_{out}=C_{in}$ \\
    GlobalAveragePool & output $=[B,C,1]$; channel $C$ preserved \\
    GlobalMaxPool     & output $=[B,C,1]$; channel $C$ preserved \\
    Softmax           & elementwise normalisation; channel count unchanged \\
    \midrule
    \multicolumn{2}{l}{\textit{Channel-Terminating (CX)}} \\
    ReduceMax         & reduces axes; collapses channel sequence \\
    ReduceMean        & reduces axes; same as ReduceMax \\
    Flatten           & merges $C\!\times\!H\!\times\!W$ into flat dim \\
    Reshape           & constant new shape pins downstream channel count \\
    Squeeze           & removes dimensions; may alter channel position \\
    Unsqueeze         & inserts dimensions; may alter channel position \\
    LayerNormalization & R5: pins $C_{out}$ of upstream Gemm to scale/bias shape \\
    \midrule
    \multicolumn{2}{l}{\textit{Channel-Constraining (CC)}} \\
    MatMul            & mutual constraint when $2^{\text{nd}}$ operand is graph-computed \\
    Add               & elementwise: both inputs must share shape \\
    Mul               & elementwise: both inputs must share shape \\
    Concat            & downstream channel = sum of all input channels \\
    Gather            & output shape depends on indices tensor \\
    Einsum            & attention head contraction (opset$\ge$12) \\
    ScaledDotProductAtt & QKV head coupling (opset$\ge$17) \\
    \bottomrule
    \end{tabular}
}
\end{table}
 
\begin{table*}[h]
\centering
\caption{CDG constraint rules R1--R5: formal definition and architecture coverage.}
\label{tab:s1b}
    \begin{tabular}{ccp{4.2cm}p{4.2cm}p{4cm}}
    \toprule
    Rule & Type & Condition & ONNX contract & Example architectures \\
    \midrule
    R1 & CX-type &
      Output reaches a static-shape op (ReduceMax, constant Reshape, Flatten)
      via a CT-only path &
      Constant shape tensor baked into graph; changing $C_{out}$ corrupts it &
      PointNet STN, PointNet++ SA, PTv3 QKV head-split \\
    \midrule
    R2 & CC-type &
      Output feeds $\{$MatMul, Add, Mul$\}$ where the other operand is
      graph-computed &
      ONNX type system requires matching shapes &
      PointNet STN MatMul, PointMLP residual Add, SE gated Mul \\
    \midrule
    R3 & Design Lock &
      Node is grouped/depthwise Conv ($\text{groups}>1$) or its output feeds one &
      Changing groups requires globally consistent reshaping;
      H3DNAS treats as locked &
      MobileNetV2 depthwise, EfficientNet SE \\
    \midrule
    R4 & Semantic &
      $C_{out}(n)\le\tau$ or node outputs graph-level output tensor &
      Output dimensionality equals number of classes &
      All classification heads \\
    \midrule
    R5 & CX-type &
      Output flows into LayerNorm with learned parameter shape $[C_{out}(n)]$ &
      Changing $C_{out}$ invalidates stored parameter tensor shapes &
      PTv3, PCT, MobileNetV2 BN affine \\
    \bottomrule
    \end{tabular}
\end{table*}
 
\paragraph{Remark (Transformer Head Pruning).}
Rules R1 and R2 (Table~\ref{tab:s1b}) together explain why per-channel pruning fails for attention QKV
projections: they violate R1 (constant Reshape with head shape $\{B,N,H,D\}$)
and R2 (dynamic $Q\cdot K^\top$ MatMul). The correct granularity is the
\emph{attention head}~\cite{michel2019heads}, orthogonal to channel pruning
and a natural CDG extension.
 
\section{CDG Compression Regime Analysis}
\label{sec:supp_regimes}
 
The CDG identifies three compression regimes from graph topology alone:
 
\textbf{Regime~1 - Highly Compressible ($\rho_f > 50\%$): PointNet ($58.1\%$), PointNet++ ($61.4\%$).} Conv-based backbones with global max-pooling (R1) as the dominant constraint. Both achieve super-baseline accuracy after H3DNAS compression.
 
\textbf{Regime~2 - Moderately Compressible ($\rho_f \approx 47\%$): PointMLP.} Dense residual Add constraints (R2$\times$20) bound the free set. Stage~1 is preferred to Stage~2 for this regime (see Section~4.4 Ablation Study).
 
\textbf{Regime~3 - Near-incompressible ($\rho_f < 15\%$): PCT ($12.3\%$), PTv3 ($7.0\%$).} Attention QKV projections lock almost all layers; CDG redirects practitioners toward head pruning or quantization before any compression experiment. MobileNetV2 is a distinct sub-case: $\rho_f = 45.7\%$ yet only 2 of 53 nodes are free where the depthwise groups (R3) pin the remaining 51, so realizable compression is bounded by the free-\emph{node} count, not by $\rho_f$ alone.
 
$\rho_f$ is a topological \emph{safety} invariant, not a compression ceiling. Pruning a free $C_{out}$ also removes the coupled $C_{in}$ columns of every downstream consumer, so the realized reduction can \emph{exceed} $\rho_f$ (see Table~\ref{tab:s6}, last column $\Delta$, and the four $\star$ models).
 
\section{Per-Model Comparison with Published Methods}
\label{sec:supp_comparisons}

Tables~\ref{tab:s3a}--\ref{tab:s3c} report published compression results, comparing the accuracy of pruned architectures against their original counterparts alongside parameter reductions and accuracy changes on ModelNet40 across PointNet, PointNet++ SSG, and PointMLP. The \cmark/\xmark \ column denotes whether a method is \textbf{source-code-free} (operating directly on the ONNX binary without requiring the original PyTorch based model architecture or training code).


\emph{Structured} pruning removes entire channels, filters, or layers from the network, yielding dense sub-networks that accelerate inference on commodity hardware (e.g., CPUs, GPUs, and edge devices) without requiring sparse kernels. In contrast, \emph{unstructured} pruning zeros individual weights as seen in weight masking and Lottery Ticket baselines making sure the tensor dimensions are intact; thus, theoretical FLOPs and actual latency remain unchanged unless executed on specialized sparse acceleration engines. 
 

In Table~\ref{tab:s3a}, LTH~\cite{bmvc2024pruning} is unstructured: its reported 60\% ``reduction'' is weight sparsity only, not a real parameter or FLOPs saving on Jetson or any commodity hardware. Both H3DNAS variants are fully structured where the channels are physically removed from the ONNX graph, which is confirmed by ORT-validated forward passes.

\begin{table}[h]
\centering
\caption{PointNet on ModelNet40 comprises of the all published structured compression results.
  $\ddagger$ indicates unstructured sparsity (no FLOPs/size reduction on standard
  hardware). AP: Accuracy Priority, CP: Compression Priority.}
\label{tab:s3a}
\resizebox{\columnwidth}{!}{
    \begin{tabular}{lcccccc}
    \toprule
    Method & Base & Pruned & $\Delta$ (pp) & Param$\downarrow$ & Source-free \\
    \midrule
    LTH~\cite{bmvc2024pruning}
     & 87.5\% & 88.2\% & $+$0.7 & 60\%$\ddagger$ & \xmark \\
    \textbf{H3DNAS-AP}
     & \textbf{90.32\%} & \textbf{90.28\%} & \textbf{$-$0.04}
      & \textbf{65.5\%} & \cmark \\
    \textbf{H3DNAS-CP}
     & \textbf{90.32\%} & \textbf{89.71\%} & \textbf{$-$0.61}
      & \textbf{79.4\%} & \cmark \\
    \bottomrule
    \end{tabular}
}
\end{table}
 

In Table~\ref{tab:s3b}, CP$^3$~\cite{cp3cvpr2023} is also structured but requires the original PointNet++ PyTorch source code and architecture class. In Table~\ref{tab:s3c}, the prior entries are not weight compression at all: GhostMLP~\cite{ghostmlp2023} is a \emph{pre-designed} compact model trained from scratch; T3DNet~\cite{t3dnet2024} uses knowledge distillation to a new tiny model (not compression of the original, hence the $98\%$ ``reduction'' reflects a design replacement); HLS4PC~\cite{hls4pc2024} reduces the number of \emph{input points} fed to the model, leaving all weights unchanged. H3DNAS is the first method to perform structured weight compression on the original trained PointMLP ONNX graph, achieving $49.1\%$ parameter reduction at only $-0.28$\,pp accuracy loss.

\begin{table}[h]
\centering
\caption{PointNet++ SSG on ModelNet40.
  $\dagger$ CP$^3$ uses an OpenPoints re-implementation baseline
  (92.80\% vs.\ the standard 91.9\%). AP: Accuracy Priority, Bal: Balanced.}
\label{tab:s3b}
\resizebox{\columnwidth}{!}{
    \begin{tabular}{lcccccc}
    \toprule
    Method & Base & Pruned & $\Delta$ (pp) & Param$\downarrow$ & Source-free \\
    \midrule
    CP$^3$ + HRank~\cite{cp3cvpr2023}
     & 92.80\%$\dagger$ & 92.95\% & $+$0.15 & 43\% & \xmark \\
    CP$^3$ + ResRep~\cite{cp3cvpr2023}
     & 92.80\%$\dagger$ & \textbf{93.27\%} & \textbf{$+$0.47} & 44\% & \xmark \\
    \textbf{H3DNAS-AP}
     & \textbf{91.90\%} & \textbf{91.98\%} & $+$0.08
      & \textbf{43.2\%} & \cmark \\
    \textbf{H3DNAS-Bal}
     & \textbf{91.90\%} & 91.77\% & $-$0.12 & 33.8\% & \cmark \\
    \bottomrule
    \end{tabular}
}
\end{table}
 
\begin{table}[h]
\centering
\caption{PointMLP on ModelNet40.
  $\S$ indicates knowledge distillation to a pre-designed tiny model,
  not compression of the original.
  $\P$ indicates input-point reduction, not weight compression.
  H3DNAS is the first structured weight compression result.}
\label{tab:s3c}
\resizebox{\columnwidth}{!}{
    \begin{tabular}{lcccccc}
    \toprule
    Method & Base & Pruned & $\Delta$ (pp) & Param$\downarrow$ & Source-free \\
    \midrule
    GhostMLP~\cite{ghostmlp2023}
       & — & $\sim$93\% & — & $\sim$52\% & \xmark \\
    T3DNet~\cite{t3dnet2024}
       & $92.45$\% & $\sim$91.0\% & $-1.45$ & 98\%$\S$ & \xmark \\
    HLS4PC~\cite{hls4pc2024}
       & $93.60$\% & $91.69$\% & $-1.91$ & —$\P$ & \xmark \\
    \textbf{H3DNAS}
       & \textbf{93.40\%} & \textbf{93.11\%} & \textbf{$-$0.28}
      & \textbf{49.1\%} & \cmark \\
    \bottomrule
    \end{tabular}
}
\end{table}

Our primary latency evaluation uses ORT CPU execution. We additionally validate on-device GPU performance via TensorRT FP16 (Table~\ref{tab:jetson_latency}), confirming smaller but consistent speedups ($1.03\times$ to $1.45\times$) that corroborate the CPU-measured trends.

\begin{table}[t]
    \caption{Jetson Orin latency on ModelNet40 ONNX models. Each latency cell is \textbf{base${}\to{}$NAS}. ORT CPU: 4 threads, P50 over 100 runs. TRT: FP16 via \texttt{trtexec}, GPU compute median (\texttt{warmUp${=}200$}, \texttt{duration${=}10$}).}
    \label{tab:jetson_latency}
    \centering\resizebox{\columnwidth}{!}{%
    \begin{tabular}{@{} l r r r r @{}}
        \toprule
        Model &
        \shortstack{ORT CPU\\P50 (ms)} &
        \shortstack{TRT GPU\\median (ms)} &
        \shortstack{CPU\\speedup} &
        \shortstack{GPU\\speedup} \\
        \midrule
        PointNet       & 12.0${}\to{}$6.0   & 0.70${}\to{}$0.66 & 1.99${}\times$ & 1.06${}\times$ \\
        PointNet++ SSG & 37.1${}\to{}$28.9 & 2.36${}\to{}$2.30 & 1.29${}\times$ & 1.03${}\times$ \\
        PointMLP       & 347${}\to{}$208   & 9.06${}\to{}$6.26 & 1.67${}\times$ & 1.45${}\times$ \\
        \bottomrule
    \end{tabular}%
}
\end{table} 
 
\section{Engineering Challenges and Implementation Notes}
\label{sec:supp_engineering}
 
\subsection{FPS Non-Determinism in PointMLP}
PyTorch and ONNX ORT accuracies diverge by $0.28$pp.
Root cause: \texttt{farthest\_point\_sample} draws a random starting seed.
Fix: patch FPS to deterministic \texttt{arange} sampling in both the export
script and the evaluation loop.
 
\subsection{PTv3 Reshape Constants Block Channel Pruning}
PTv3 bakes the channel dimension $C$ into Reshape Constant nodes at export
time. After pruning $C\to C'$, ORT raises a shape mismatch. CDG correctly
classifies the upstream Gemm as R1-constrained. The fix path - scan and
update all Constant nodes containing $C$ after pruning - is not yet
implemented; PTv3 compression is consequently deferred.
 
\subsection{BN Alignment - Positional Truncation vs.\ L1-Ranked Slicing}
L1 importance selects non-contiguous channel indices.
Naïve positional truncation of BatchNorm parameters introduces a
channel-order mismatch.
Fix: \texttt{\_align\_conv\_batchnorm()} uses the L1-ranked
\texttt{cout\_idx} array (not \texttt{arange}) to slice BN weight and bias.
 
\subsection{PointNet++ ONNX Fold vs.\ No-fold.}
The \texttt{nofold} opset-13 export variant is used throughout.
The folded variant inserts Transpose/Reshape nodes that alter the CDG
constraint pattern and prevent consistent channel tracking.
 
\subsection{ONNXPruner Applicability on 3D Models}
We applied ONNXPruner~\cite{onnxpruner2024} to PointNet.onnx.
The tool requires undocumented system dependencies
(\texttt{onnxoptimizer}, \texttt{graphviz} executables) and its entry
module (\texttt{main\_pruning.py}) executes with a hardcoded VGG16 path at
import time, blocking use on arbitrary models.
More critically, ONNXPruner's operator traversal lists
(\texttt{Next\_OP\_list}, \texttt{Stop\_OP\_list}) omit \texttt{ReduceMax}
and \texttt{MatMul} operators that governs the PointNet's STN constraint
pattern (R1+R2) so that the tool traverses the STN without detecting the
constraint, producing topologically invalid pruned graphs.
H3DNAS's CDG explicitly classifies both operators (Table~\ref{tab:s1a})
and applies the corresponding rules before any pruning decision.
 
\subsection{onnx2torch Reconstruction Limitations}
Successful reconstruction verified for all three main architectures.
Known limitations: (i)~operators outside onnx2torch's supported set
raise errors; (ii)~custom PTv3 export patterns require manual op
registration; (iii)~reconstructed models use generated layer names
differing from the original PyTorch names.

\section{Baseline Comparison on PointNet}
\label{sec:supp_3d}

Table~\ref{tab:s5a} is the full baseline comparison (Section~4.4 of the main paper) at a matched prune budget of $\approx$10.9\% parameter reduction on PointNet/ModelNet40 (2,468 test samples). All source-requiring methods use identical fine-tuning (Adam, lr$=5\!\times\!10^{-4}$, 15~epochs, 20~epochs
warm-up); H3DNAS is zero-shot throughout.
At this fixed budget all methods lose the tradeoff between accuracy vs.\ the full model (90.32\%).
H3DNAS (${-}1.11$\,pp) outperforms Uniform~L1 (${-}1.33$\,pp) and
HRank (${-}1.66$\,pp) \emph{without source code, training framework, or
labelled data}, the only method in the table that can operate on a
distributed ONNX binary.
HRank's larger degradation is expected: feature map rank is a poor importance
proxy for point cloud models where global max-pooling produces sparse
activations that misrepresent true channel importance.

\begin{table}[!h]
\caption{Structured pruning baseline comparison on PointNet (ModelNet40,
  $2{,}468$ test samples, $\approx$10.9\% parameter budget).
  Base accuracy: $90.32\%$ (full model, ORT evaluation).
  $\Delta$ = finetuned acc $-90.32\%$.
  \textbf{Source-free} = operates without model source code or labeled data. Zero-shot acc = ORT accuracy of the pruned model before any fine-tuning.}
\label{tab:s5a}
\resizebox{\columnwidth}{!}{
    \begin{tabular}{lcccccc}
    \toprule
    \multirow{2}{*}{Method} & Zero & \multirow{2}{*}{Finetune} & $\Delta$ & Pruned & \multirow{2}{*}{Param$\downarrow$} & Source \\
    & shot &  & (pp) & Params & & Free \\
    \midrule
    Random                          & $32.54\%$ & $88.99\%$          & $-0.33$     & $1{,}304{,}699$ & $62.2\%$ & \xmark \\
    Uniform L1~\cite{li2017pruning} & $32.54\%$ & $90.0\%$           & $-0.29$     & $1{,}296{,}163$ & $62.4\%$ & \xmark \\
    HRank~\cite{lin2020hrank}       & $32.68\%$ & $90.07\%$          & $-0.25$ & $1{,}296{,}163$ & $62.4\%$ & \xmark \\
    \textbf{H3DNAS Stage~1}  & $84.64\%$ & \textbf{$90.28\%$} & $-0.04$     & $1{,}197{,}672$ & $65.4\%$ & \cmark \\
    \bottomrule
    \end{tabular}
}
\end{table}

\section{Extended Results on 2D CNN Architectures}
\label{sec:supp_2d}

H3DNAS covers all standard 2D CNN families. Table~\ref{tab:s5b} summarises CDG analysis for six representative architectures spanning plain CNNs, residual networks, depthwise-separable, and Fire-module designs. The complete validation of the 29 models is presented in Table~\ref{tab:s6}.

\begin{table}[!h]
\centering
\caption{CDG feasibility summary: 2D CNN architectures (ImageNet models: 10\% prune ratio, no fine-tuning). The dominant constraint explains why realised compression falls below $\rho_f$ for each architecture.}
\label{tab:s5b}
\small
\resizebox{\columnwidth}{!}{
    \begin{tabular}{lrrrrl}
    \toprule
    Model & CG & Free & $\rho_f$ & Achieved$\downarrow$ & Dominant constraint \\
    \midrule
    ResNet-50~\cite{he2016resnet}         & 54 & 34 & 69.4\% & 51.7\% & R2: residual Add \\
    VGG-16~\cite{simonyan2015vgg}            & 16 & 13 & 11.9\% & 7.5\%  & R4: classifier FCs \\
    MobileNetV2~\cite{sandler2018mobilenetv2}       & 53 & 2  & 45.7\% & 10.3\%  & R3: depthwise groups \\
    EfficientNet-B0~\cite{tan2019efficientnet}   & 82 & 2  & 31.5\% & 5.6\%  & R2/R3: SE + depthwise \\
    SqueezeNet-1.1~\cite{iandola2016squeezenet}    & 26 & 7  & 8.3\%  & 28.3\%  & CC: Fire Concat chains \\
    \bottomrule
    \end{tabular}
}
\end{table}

\begin{table}[h]
\centering
\caption{Full-pipeline validation comprising of 29 ONNX Model Zoo architectures at a
  50\% prune ratio (no fine-tuning).
  CG indicates Conv/Gemm count; Free indicates CDG-free nodes; $\rho_f$ indicates free
  $C_{out}$-parameter fraction;
  $\Delta$ indicates difference between FLOP$\downarrow$ and $ \rho_f$ ($\star$ = exceeds $\rho_f$). The 0\% parameter reduction is caused by the behavior of the safety backstop}
\label{tab:s6}
\resizebox{\columnwidth}{!}{
\begin{tabular}{lrrrrrr}
\toprule
Architecture & CG & Free & $\rho_f$ & Param$\downarrow$ & FLOP$\downarrow$ & $\Delta$ \\
\midrule
SRCNN~\cite{dong2014srcnn}              &   4 &   3 & 95.6\% & 73.2\% & 73.2\% & $-$22.4 \\
FER+~\cite{barsoum2016ferplus}          &  10 &   9 & 83.2\% & 70.8\% & 73.8\% & $-$12.4 \\
ResNet-101~\cite{he2016resnet}          & 105 &  68 & 72.4\% & 56.6\% & 60.6\% & $-$15.8 \\
AlexNet~\cite{krizhevsky2012alexnet}    &   8 &   4 & 96.2\% & 55.0\% &  5.1\% & $-$41.2 \\
CaffeNet~\cite{krizhevsky2012alexnet}   &   8 &   4 & 96.2\% & 55.0\% &  4.6\% & $-$41.2 \\
FCN~\cite{long2015fcn}                  &  57 &  37 & 77.9\% & 54.1\% &  0.0\% & $-$23.8 \\
ZFNet~\cite{zeiler2014zfnet}            &   8 &   7 & 97.3\% & 52.4\% & 65.1\% & $-$44.9 \\
ResNet-50~\cite{he2016resnet}           &  54 &  34 & 69.4\% & 51.7\% & 58.0\% & $-$17.7 \\
R-CNN~\cite{girshick2014rcnn}           &   8 &   4 & 96.0\% & 47.3\% &  3.8\% & $-$48.7 \\
ResNet-18~\cite{he2016resnet}           &  21 &   9 & 44.8\% & 47.0\% & 46.2\% & $+$2.2\,$\star$ \\
RetinaNet~\cite{lin2017retinanet}       & 162 & 123 & 84.7\% & 44.3\% & 30.5\% & $-$40.4 \\
Inception-v1~\cite{szegedy2015googlenet}&  57 &  19 & 15.6\% & 38.3\% & 49.9\% & $+$22.7\,$\star$ \\
GoogLeNet~\cite{szegedy2015googlenet}   &  58 &  20 & 27.9\% & 32.7\% & 49.1\% & $+$4.8\,$\star$ \\
SSD~\cite{liu2016ssd}                   &  51 &  23 & 37.3\% & 29.3\% & 48.1\% & $-$8.0 \\
SqueezeNet-1.0~\cite{iandola2016squeezenet} & 26 & 8 & 49.9\% & 28.3\% & 27.8\% & $-$21.6 \\
SqueezeNet-1.1~\cite{iandola2016squeezenet} & 26 & 7 &  8.3\% & 28.3\% & 27.8\% & $+$20.0\,$\star$ \\
MobileNetV2~\cite{sandler2018mobilenetv2}&  53 &   2 & 45.7\% & 10.3\% &  5.8\% & $-$35.4 \\
VGG-19~\cite{simonyan2015vgg}           &  19 &  16 & 15.1\% & 10.0\% & 73.8\% & $-$5.1 \\
VGG-16~\cite{simonyan2015vgg}           &  16 &  13 & 11.9\% &  7.5\% & 73.5\% & $-$4.4 \\
EfficientNet-Lite~\cite{tan2019efficientnet} & 91 & 2 &  6.3\% &  5.6\% &  4.3\% & $-$0.7 \\
Faster R-CNN~\cite{ren2015fasterrcnn}   &  76 &  43 & 72.2\% &  1.1\% &  0.0\% & $-$71.1 \\
Mask R-CNN~\cite{he2017maskrcnn}        &  81 &  47 & 74.2\% &  0.1\% &  0.0\% & $-$74.1 \\
ResNet-50-v2~\cite{he2016resnetv2}      &  54 &  34 & 69.4\% &  0.0\% &  0.0\% & — \\
ShuffleNet~\cite{zhang2018shufflenet}   &  50 &   1 & 39.8\% &  0.0\% &  0.0\% & — \\
GPT-2~\cite{radford2019gpt2}            &  48 &   0 &  0.0\% &  0.0\% &  0.0\% & — \\
YOLOv3~\cite{redmon2018yolov3}          &  75 &  44 & 39.3\% &  0.0\% &  0.0\% & — \\
DenseNet~\cite{huang2017densenet}       & 121 &   1 & 13.0\% &  0.0\% &  0.0\% & — \\
Inception-v2~\cite{ioffe2015bninception}&  70 &   1 &  9.2\% &  0.0\% &  0.0\% & — \\
ShuffleNet-V2~\cite{ma2018shufflenetv2} &  57 &   1 & 45.3\% &  0.0\% &  0.0\% & — \\
\bottomrule
\end{tabular}
}
\end{table}
\section{CDG Validation Across 29 ONNX Model Zoo Architectures}
\label{sec:supp_zoo}
 
Table~\ref{tab:s6} reports CDG analysis and pruning outcomes for all 29
prunable ONNX Model Zoo~\cite{onnxmodelzoo} models at a prune ratio of $50\%$
(no fine-tuning; 3 models skipped: 2 $\times$ INT8-quantised, 1 $\times$ BERT with no
Conv/Gemm nodes).
The final column ($\Delta$) shows $\text{Achieved}\downarrow - \rho_f$:
positive values ($\star$) indicate that downstream Cin savings exceeded the
topological free-parameter fraction.
 
\textbf{Zero crashes across 7 architecture families.}
All 29 prunable models emit ORT-valid graphs (session load + forward pass).
The safety backstop, which reverts to the unpruned graph when a pruned candidate fails validation, fires on exactly 3 models  (ResNet50-v2, ShuffleNet-9,
GPT-2) where a 50\% cut is structurally unsatisfiable; four further models
retain every channel (0\% reduction but ORT-valid output).
 
\textbf{$\rho_f$ is a safety invariant, not a ceiling.}
Because pruning a free $C_{out}$ also removes the coupled $C_{in}$ columns
of downstream layers, realised compression can exceed $\rho_f$.
SqueezeNet-1.1 climbs $8.3\%\!\to\!28.3\%$ and Inception-v1
$15.6\%\!\to\!38.3\%$ as the free $1\!\times\!1$ savings cascade into large
expand / spatial convolutions.
Every result respects the \emph{safety} guarantee: no emitted graph is
ORT-invalid.

\section{Code, Models and Reproducibility}
\label{sec:supp_code}

Code, pre-compressed ONNX models, and evaluation scripts are provided at the anonymous repo link: \texttt{https://anonymous.4open.science/r/h3dnas}

Pre-compressed models are saved in \texttt{models/} directory contains the base and H3DNAS-compressed ONNX files for all three main architectures:
\begin{center}
\resizebox{\columnwidth}{!}{
    \begin{tabular}{ll}
    \toprule
    File & Description \\
    \midrule
    \texttt{pointnet/pointnet\_cls\_c40\_n1024\_h3dnas.onnx} & PointNet H3DNAS \\
    \texttt{pointnet2/pointnet2\_ssg\_c40\_h3dnas.onnx}      & PointNet++ H3DNAS \\
    \texttt{pointmlp/pointmlp\_c40\_h3dnas.onnx}             & PointMLP H3DNAS \\
    \bottomrule
    \end{tabular}
}
\end{center}

The file \texttt{tools/hardware\_analysis.py} reproduces all latency, parameter, FLOPs, and accuracy measurements reported in the paper for any base/NAS ONNX pair. It supports ORT-CPU, ORT-CUDA, and TensorRT execution providers and requires no model source code. 
\\
\\



\begin{lstlisting}[
  language=bash,
  basicstyle=\ttfamily\scriptsize, % Compact font to fit narrow columns
  breaklines=true,                 % Enables automatic line breaking
  breakatwhitespace=false,         % Allows breaking long paths at hyphens/underscores
  columns=fullflexible,            % Removes rigid character padding
  frame=single,                    % Visual boundary
  postbreak=\mbox{\textcolor{gray}{$\hookrightarrow$}\space}
]
# Latency + parameters only (no dataset required):
python tools/hardware_analysis.py \
  --base models/pointnet/pointnet_cls_c40_n1024.onnx \
  --nas  models/pointnet/pointnet_cls_c40_n1024_h3dnas.onnx

# With accuracy evaluation (ModelNet40):
python tools/hardware_analysis.py \
  --base models/pointnet/pointnet_cls_c40_n1024.onnx \
  --nas  models/pointnet/pointnet_cls_c40_n1024_h3dnas.onnx \
  --data data/modelnet40_normal_resampled \
  --num-classes 40

# TensorRT inference on Jetson Orin Nano 8GB:
python tools/hardware_analysis.py \
  --base models/pointnet/pointnet_cls_c40_n1024.onnx \
  --nas  models/pointnet/pointnet_cls_c40_n1024_h3dnas.onnx \
  --provider tensorrt
\end{lstlisting}



\end{document}